\documentclass[11pt]{article}

\usepackage{mystyle}
\usepackage{tikz}
\usetikzlibrary{arrows.meta}
\usepackage{fullpage}
\usepackage{etoolbox}

\makeatletter
\patchcmd{\@maketitle}{\vskip 2em}{\vspace*{-2.5em}}{}{\PackageError{main}{Title spacing adjustment failed}{}}
\makeatother

\title{\LARGE Minimax-Optimal Online Contract Design
with Unrestricted Bounded Contracts}
\author{Rui Ai\thanks{Massachusetts Institute of Technology. Email: \texttt{ruiai@mit.edu}.}
\and David Simchi-Levi\thanks{Massachusetts Institute of Technology. Email: \texttt{dslevi@mit.edu}.}
\and Han Zhong\thanks{Shanghai Jiao Tong University. Email: \texttt{han.zhong@sjtu.edu.cn}.}}
\date{}

\begin{document}
\maketitle

\begin{abstract}
We study repeated contract design when a principal observes outcomes but not the actions that generate them. The principal may use any bounded outcome-contingent payment vector, and the agent's best response can make expected profit discontinuous in those payments.
For every fixed number $m\ge2$ of outcomes, the minimax
regret over $T$ rounds is of order $T^{m/(m+1)}$,
up to logarithmic factors.
The upper bound allows arbitrary action spaces and agent heterogeneity, without smoothness or monotone-surplus assumptions.
Its key is an effective-dimension reduction that the benchmark can be normalized even when fixed tie-breaking is not shift invariant, after which revealed preference yields a monotone response map in payment-difference coordinates.
A learning policy built on a Lipschitz parametrization of this map attains the rate using only observed outcome categories. 
The lower-bound construction accounts for how incentive losses accumulate across outcome dimensions. It shows that each additional contractible outcome creates a precise and unavoidable increase in the worst-case cost of learning.
\end{abstract}


\section{Introduction}

Outcome-contingent pay is most useful when effort is hidden. A platform may pay a provider by customer-rating category, or a buyer may condition procurement payments on a quality grade. When the environment is unknown, the payment schedule must both create incentives and generate information. This learning problem is unusually irregular in that a small payment change can switch the agent's hidden action and discontinuously change the principal's profit, so neither smooth optimization nor a naive discretization of payment space provides a general solution.

We characterize the minimax regret of this unrestricted finite-outcome problem up to logarithmic factors. For every fixed number of outcome categories, the sharp exponent is the number of outcomes divided by one plus that number, while with one outcome, optimal regret is zero. The upper bound permits heterogeneous agents, arbitrary action spaces, and discontinuous responses. The matching lower bound already holds with one type, finitely many actions, bounded nonnegative costs, and a common deterministic tie-breaking rule. Thus, richer outcome classifications have a precise worst-case learning cost even though they may make contracts more expressive.
The improvement is quantitative even in the smallest nontrivial case. For two outcomes, the closest general upper bound of \citet{zhu2023sample} scales as $\widetilde O(T^{4/5})$, whereas our sharp rate is $\widetilde\Theta(T^{2/3})$. For general fixed $m$, the upper exponent improves from $1-1/(2m+1)$ to $1-1/(m+1)=m/(m+1)$. The gain comes without restricting the number of actions or imposing response regularity.

Three contributions deliver this characterization. First, we close the upper-lower gap for the unrestricted fixed-dimensional model under a single convention in which $m$ counts all outcomes, including any null outcome. Second, we expose a reduced monotone geometry and turn it into an outcome-feedback policy that does not know the action set, costs, type distribution, outcome laws, or response map. Third, we give an aggregate-response lower-bound construction that works in every outcome dimension and uses one common deterministic tie-breaking rule across all hard instances. Together, these results identify both the optimal statistical rate and the mechanism that determines it.

The upper bound rests on an effective-dimension reduction. A common shift of all payments preserves every agent's best-response set but may change which tied action a fixed selector chooses. We therefore cannot simply quotient the selected response by common shifts. A perturbation argument instead shows that the benchmark supremum can be restricted to contracts with minimum payment zero without assuming shift-invariant selection. In the resulting payment-difference coordinates, revealed preference makes the selected outcome law monotone. Adding the contract coordinate to this law produces a \emph{Minty coordinate} that parametrizes the selected-response graph in a Lipschitz way, even though the contract-to-profit map can jump.
This geometry changes what the learner discretizes. Rather than cover the discontinuous profit function, the policy covers a known ambient cube of Minty coordinates. Each grid point drives a projected stochastic-approximation simulator using only the realized outcome category, and a rested upper-confidence master allocates rounds among simulators. A simulator near the optimal graph coordinate has uniformly small prefix regret, and the master competes with it. Balancing graph approximation with statistical allocation yields the upper rate.

The distinction between contract space and graph space is central. An ordinary payment grid can approach a near-optimal contract yet induce a different action and profit. Our grid instead approximates Minty coordinates of reduced contracts and their selected outcome laws. Stochastic approximation of the strongly monotone residual yields a uniform prefix-regret certificate that survives the rested master's adaptive allocation. This is why discontinuity of the original profit function does not prevent learning at the sharp rate.

For the lower bound, we build many finite-action alternatives that differ from a baseline only through one action cost. The distinguished action is selected according to its \emph{aggregate} payoff deficit across coordinates. These aggregate response regions are disjoint, contain contracts with a controlled profit improvement, and localize the statistical information about their associated alternatives. A change-of-measure argument then gives the same exponent as the upper bound.
This construction also resolves a dimension issue in the closest unrestricted benchmark. \citet{zhu2023sample} establish the binary lower bound and a general upper bound, but their multidimensional appendix uses more outcome coordinates than the main text's dimension convention and a product response cell that treats coordinatewise losses separately even though those losses accumulate while the distinguished action receives its cost reduction only once. That calculation therefore does not establish the stated multidimensional exponent when dimension counts total outcomes. Our self-contained aggregate-region construction avoids this step and we give the exact comparison in Section~\ref{sec:lower}.

The selected response in our model is fixed before learning and need not be known to the principal. This assumption is automatic under a unique best response and is implemented by any predetermined priority rule under ties. It excludes only history-dependent adversarial tie-breaking, which would make the response law itself nonstationary. The policy observes the outcome category used to settle the contract, not hidden actions, types, costs, or outcome distributions. It requires neither ordered outcomes, monotone effort, smooth behavior, nor a parametric response model.

The operational message is two-sided. Refining a performance classification can enlarge the set of incentives a principal can express, but without behavioral structure it also raises the worst-case cost of finding a good contract. 
On the other hand, our normalization reduces the $m$ payment coordinates to $m-1$ independent payment differences. The minimax rate quantifies how this dimension affects worst-case learning, while the contracting benefit of finer outcome classifications depends on the application.

\begingroup
\emergencystretch=2em
\paragraph{Related work.}
Classical hidden-action theory studies known primitives \citep{holmstrom1979moral,grossman1983analysis}, whereas algorithmic contract design asks how actions, outcomes, and private types affect computation, approximation, and statistical complexity. \citet{guruganesh2021contracts} study moral hazard together with adverse selection, while \citet{duetting2025pseudo} characterize statistical complexity through the pseudo-dimension of contract classes.
\par\endgroup

Online contract design instead requires the principal to learn from repeated outcomes. \citet{ho2016adaptive} use bandit methods and adaptive discretization in crowdsourcing markets.
\citet{cohen2023learning} study how to learn monotone-smooth contracts for identical agents under stochastic dominance and bounded risk aversion.
The closest unrestricted benchmark is \citet{zhu2023sample}. We retain its finite-outcome model, and our improvement comes from reduced monotone geometry rather than additional response regularity. Complementary results also exploit structure absent here, e.g., \citet{bacchiocchi2025learning} obtain polynomial sample complexity and improved regret when the number of actions is fixed, while \citet{zuo2025continuous} studies continuous actions under first-order or Lipschitz regularity.
\citet{chen2024bounded} learn near-optimal bounded contracts using polynomially many queries under first-order stochastic dominance and diminishing returns to effort.
\citet{bacchiocchi2025approximate} study contract learning with approximate best responses and evaluate each contract by the principal's worst expected payoff among those responses.
Our fresh-agent stationary-response model is also distinct from repeated contracting with a persistent no-regret learning agent \citep{guruganesh2024contracting}. Methodologically, we combine monotone-operator ideas \citep{minty1962monotone,rockafellar1976monotone}, stochastic approximation \citep{robbins1951stochastic}, martingale concentration \citep{freedman1975tail}, and upper-confidence allocation \citep{auer2002finite}.

The paper proceeds as follows. Sections~\ref{sec:model}-\ref{sec:proof-main} give the model, geometry, algorithm, and upper bound. Section~\ref{sec:lower} develops the lower bound, and we defer all remaining proofs to the appendices.

\section{Model and Learning Objective}\label{sec:model}

We assume there are $m<\infty$ outcomes. The principal's public value vector is $v\in[0,1]^m$, and a contract is a payment vector $f\in[0,1]^m$. Write $\Delta_m=\{p\in\mathbb R_+^m:\sum_jp_j=1\}$.

\begin{assumption}[Selected best responses]\label{ass:selected}
At each round, a type $\theta$ is drawn independently from a fixed distribution. Type $\theta$ has an action set $A_\theta$. Action $a$ has outcome law $p_{\theta,a}\in\Delta_m$ and cost $c_{\theta,a}\in\mathbb R$. For every $\theta$ and $f$, a selected response
\[
    a_\theta(f)\in\argmax_{a\in A_\theta}\{p_{\theta,a}\cdot f-c_{\theta,a}\}
\]
exists, is fixed independently of the learner's history, and induces a measurable map $(\theta,f)\mapsto p_{\theta,a_\theta(f)}$.
\end{assumption}
Assumption~\ref{ass:selected} is mild in practice in that, for instance, any finite action set with a fixed measurable priority rule satisfies it. The assumption imposes no stochastic ordering, smoothness, response continuity, or separate participation constraint.
For finite action sets, fixing an action priority before learning makes the selected response automatic. For infinite action spaces, the assumption records only the existence of a maximizer and the measurability of the induced outcome law. In particular, it does not require the selection to be invariant when the same constant is added to every payment. That distinction matters for normalization and is handled explicitly in Lemma~\ref{lem:normalization}.

Let $P(f)=\mathbb E_\theta[p_{\theta,a_\theta(f)}]\in\Delta_m$. After posting $f$ given history $\mathcal H$, the principal observes the realized outcome as a one-hot vector $X$ satisfying
\begin{equation}\label{eq:conditional-mean}
    \mathbb E[X\mid f,\mathcal H]=P(f).
\end{equation}
This is outcome-category feedback, as in \citet{zhu2023sample} where the learner does not observe $P(f)$, the type, or the action. Fresh types and outcomes make~\eqref{eq:conditional-mean} valid under adaptive policies.

The feedback is economically natural as the outcome category must be observed in order to settle an outcome-contingent contract. It is nevertheless stronger than observing only the realized scalar profit. The algorithm uses the one-hot vector as an unbiased estimate of the selected outcome law, so the result should not be read as a scalar bandit-feedback guarantee.

The principal's expected one-period utility is $u(f)=P(f)\cdot(v-f)$. A policy chooses $f_t$ from past contracts and outcomes, and its expected Stackelberg regret is
\[
    \mathfrak R_T=T\sup_{f\in[0,1]^m}u(f)
    -\mathbb E\!\left[\sum_{t=1}^T u(f_t)\right].
\]
The learner knows $m$, $v$, and the protocol, but none of the remaining primitives or the induced map $P$.

Let $\mathcal I_m$ be the class of instances above and define
$\mathfrak R_T^\star(m)=\inf_\pi\sup_{I\in\mathcal I_m}\mathfrak R_T(\pi,I)$ over adaptive outcome-feedback policies. When $m=1$, $P(f)=1$ and $u(f)=v_1-f$, so posting $f=0$ gives zero regret. Henceforth we study $m\ge2$ different outcomes.

The upper bound applies to the full class $\mathcal I_m$, including heterogeneous types and infinite action spaces, whereas the lower bound uses one type, finitely many actions, bounded nonnegative costs, and common deterministic tie-breaking. 
Constants $C_m<\infty$ may change from line to line and depend only on fixed
$m$. Throughout, a subscript $m$ on asymptotic notation allows the hidden
constant to depend on the fixed number of outcomes $m$. A tilde additionally
suppresses factors polylogarithmic in $T$. The dimension-only algorithm constant $B_m$ is fixed before learning and chosen sufficiently large for the concentration certificates below.

This asymmetry between the two theorem classes is intentional. The upper bound is robust to essentially all hidden primitives once the selected response is stationary, whereas the lower bound shows that neither heterogeneity, infinite action spaces, nor instance-dependent tie-breaking is responsible for the difficulty. The common hard-family structure will also make the information comparison transparent.

\section{Main Results}

The headline result is an exact fixed-dimension minimax exponent. The two theorems below use the same convention that $m$ is the total number of outcomes, including any null outcome. The upper theorem states the finite-scale guarantee needed to see every source of regret, while the lower theorem shows that the leading exponent survives on a sharply restricted subclass.

For a resolution parameter $\varepsilon\in(0,1)$, the upper-bound
policy uses a grid of $K_\varepsilon$ candidate contract-response
coordinates. Section~\ref{sec:algorithm} defines this grid formally.

\begin{theorem}[Fixed-scale upper bound for unrestricted contracts]\label{thm:main}
For every fixed $m\ge2$, there exists a finite constant $B_m$, depending only on $m$, such that, for every $T\ge2$ and $\varepsilon\in(0,1)$, Algorithms~\ref{alg:simulator}-\ref{alg:master} with confidence $\delta=T^{-2}$ use only $m$, $v$, $T$, and $\varepsilon$ and satisfy, on every instance in $\mathcal I_m$,
\[
    \mathfrak R_T\le C_m\big(T\varepsilon+L_T\sqrt{TK_\varepsilon}+K_\varepsilon\big),
\]
where $K_\varepsilon\le(1+3\sqrt{m-1}/\varepsilon)^{m-1}$ and $L_T=1+\log^3(16K_\varepsilon T^3)$. Consequently, $\varepsilon\asymp T^{-1/(m+1)}$ up to logarithmic factors gives $\mathfrak R_T\le\widetilde O_m(T^{m/(m+1)})$.
\end{theorem}

\begin{proof}[Proof sketch]
    The upper bound proof has three ingredients.
    \begin{enumerate}
    \item \emph{Normalization without response invariance.} A common payment shift preserves each best-response set but can change the action chosen from a tie. A perturbation toward the value vector normalizes the benchmark supremum without assuming that selected responses themselves are shift invariant.
    \item \emph{Regularizing the selected-response graph.} Revealed preference makes the reduced outcome law monotone. The Minty transformation adds the contract coordinate to that law, making both components Lipschitz functions of the transformed graph coordinate despite discontinuity in payments.
    \item \emph{Learning an unknown graph.} The policy covers a known ambient cube, assigns a projected stochastic-approximation simulator to every grid point, and uses a rested-UCB master. One near-graph simulator has a uniform prefix certificate, and the master can find it without knowing its identity.
\end{enumerate}

The three terms have separate roles. The term $T\varepsilon$ is the price of approximating the optimal selected-response graph point, $L_T\sqrt{TK_\varepsilon}$ is the cost of allocating observations across rested simulators, and $K_\varepsilon$ activates the grid. %
Since a grid of resolution $\varepsilon$ has order
$\varepsilon^{-(m-1)}$ points, balancing the approximation cost
$T\varepsilon$ and the statistical allocation cost
$\widetilde O_m(\sqrt{T\varepsilon^{-(m-1)}})$ gives
$\varepsilon\asymp T^{-1/(m+1)}$ and regret
$T^{m/(m+1)}$.
See Sections~\ref{sec:geometry}-\ref{sec:proof-main} and the
appendices for the detailed proof.
\end{proof}

\begin{theorem}[Matching finite-outcome lower bound]\label{thm:lower}
For every fixed $m\ge2$, there is $c_m>0$ such that every adaptive outcome-feedback policy and every $T\ge1$ incur
$\mathfrak R_T\ge c_mT^{m/(m+1)}$ on some finite-action, single-type instance in $\mathcal I_m$. The hard family has public value $(1,\ldots,1,0)$, nonnegative costs bounded by one, the same finite action set and action-level outcome laws, and one common deterministic action-priority rule while each alternative changes only one action cost from a common baseline.
\end{theorem}

\begin{proof}[Proof sketch]
Each lower-bound alternative discounts one action, whose selection
depends on its aggregate payoff deficit across coordinates. The construction packs
$\Theta(\varepsilon^{-(m-1)})$ disjoint aggregate response regions, and a
visit to one region reveals only $O(\varepsilon^2)$ information about
its associated alternative. Averaged across the packing, the
information scale is $O_m(T\varepsilon^{m+1})$. Keeping it bounded forces
$\varepsilon\asymp T^{-1/(m+1)}$ and produces regret
$\Omega_m(T\varepsilon)=\Omega_m(T^{m/(m+1)})$.
See Section~\ref{sec:lower} and Appendix~\ref{app:lower} for the
detailed proof.
\end{proof}

The lower bound is witnessed by a deliberately parsimonious hard family. Every instance has one agent type and uses the same finite action set, action-level outcome laws, and deterministic tie-breaking priority. All costs are bounded and nonnegative, and each alternative changes only one action cost relative to a common baseline. The construction therefore pins the sharp $T^{m/(m+1)}$ exponent on the statistical difficulty of learning how a localized incentive perturbation changes contract-induced behavior from outcome feedback.

Combining Theorems~\ref{thm:main} and~\ref{thm:lower}, we obtain, for
every fixed $m\ge2$,
\begin{equation}\label{eq:headline-rate}
    \mathfrak R_T^\star(m)=\widetilde\Theta_m(T^{m/(m+1)}).
\end{equation}

We next discuss how the result improves the existing bounds and the scope of its operational interpretation.

\vspace{5pt}
\noindent\textbf{What closes the previous gap?}
For unrestricted contracts, \citet{zhu2023sample} give the valid general upper bound $\widetilde O(\sqrt m\,T^{1-1/(2m+1)})$. Their discretization covers directions and radii in payment space, whereas our policy covers the reduced selected-response graph. After normalization, this graph has $m-1$ effective coordinates, and the Minty transformation makes graph approximation control utility even though payment-space approximation does not. The resulting cover contributes order $\varepsilon^{-(m-1)}$, rather than the larger payment-space discretization that drives the previous exponent. This yields $T^{m/(m+1)}$ and, together with our aggregate-region multidimensional lower bound, closes the minimax exponent.

The improved upper bound requires both normalization and Minty
regularization. Normalization alone removes the common-payment direction but leaves a discontinuous objective, while Minty regularization alone, without normalization, would retain an unnecessary dimension and miss the sharp exponent. On the lower side, a coordinatewise product cell would not match this reduced geometry because payoff deficits add across coordinates. The aggregate-response regions repair that accumulation and make the same $m-1$ dimensions appear in the testing bound. Thus the upper and lower constructions identify one common geometric obstruction rather than producing coincidentally matching powers.

\vspace{5pt}
\noindent\textbf{Operational interpretation and scope.} \quad
Operationally, the rate prices the number of independently contractible outcome differences, not payment levels. The policy observes only the outcome category and is information-theoretic rather than dimension-free computationally that its grid is polynomial in $T$ for fixed $m$ but exponential in $m$. Structured responses or contract families may permit faster rates, but scalar reward-only feedback does not provide the outcome signal used here.

For an organization choosing how finely to classify performance, our theorems isolate a worst-case cost of granularity. Splitting an outcome category may improve the best contract when primitives are known, but it also adds a payment-difference direction that must be learned. The result quantifies the second effect and deliberately leaves the first, application-specific benefit to a separate model-selection problem.

\section{Normalization and Reduced Minty Geometry}\label{sec:geometry}

The principal's utility may jump at response boundaries, so proximity of two contracts alone says little about proximity of their profits. The optimizing behavior nevertheless restricts the direction of these jumps. We first extract that revealed-preference restriction, then remove the economically redundant common-payment direction, and finally use a Minty coordinate to turn graph proximity into utility proximity.

\begin{lemma}[Revealed-preference monotonicity]\label{lem:monotone}
For any $f,g\in[0,1]^m$, $(P(f)-P(g))\cdot(f-g)\ge0$.
\end{lemma}

\begin{proof}
For a fixed type, let $(p_f,c_f)$ and $(p_g,c_g)$ be the outcome-cost pairs selected under $f$ and $g$. Optimality gives $p_f\cdot f-c_f\ge p_g\cdot f-c_g$ and $p_g\cdot g-c_g\ge p_f\cdot g-c_f$. Adding and averaging over types proves the claim.
\end{proof}

Lemma~\ref{lem:monotone} does not assert continuity. It says that response jumps are ordered in aggregate. The Minty transformation below converts precisely this weak order into metric control.

\begin{lemma}[Benchmark normalization]\label{lem:normalization}
For every instance satisfying Assumption~\ref{ass:selected},
\[
    \sup_{f\in[0,1]^m}u(f)=
    \sup_{f\in[0,1]^m:\,\min_jf_j=0}u(f),
\]
without requiring the selected response to be invariant under common payment shifts.
\end{lemma}

\begin{proof}
Fix $f$, let $c=\min_i f_i$, and set $g=f-c\mathbf1$. Choose $j\in\argmin\{v_i:g_i=0\}$ and, for sufficiently small $\eta>0$, set $h_\eta=(1-\eta)g+\eta(v-v_j\mathbf1)$. Then $h_\eta\in[0,1]^m$ and $\min_i h_{\eta,i}=0$, coordinate $j$ is zero, other zero coordinates of $g$ are nonnegative by the choice of $j$, and positive coordinates remain nonnegative.

For each type, let $p_\theta$ be the law selected at $f$ and $q_\theta$ the law selected at $h_\eta$. The action selected at $f$ remains optimal at $g$ because $f-g=c\mathbf1$. Comparing it with the action selected at $h_\eta$ in both directions gives $(q_\theta-p_\theta)\cdot(h_\eta-g)\ge0$, hence $q_\theta\cdot(v-g)\ge p_\theta\cdot(v-g)$ because both laws have unit mass. Therefore,
\[
u(h_\eta)=(1-\eta)\mathbb E_\theta[q_\theta\cdot(v-g)]+\eta v_j
\ge(1-\eta)(u(f)+c)+\eta v_j.
\]
Letting $\eta\downarrow0$ and then taking the supremum over $f$ proves the nontrivial direction, and we immediately obtain the reverse.
\end{proof}

The perturbation is needed because subtracting $c\mathbf 1$ preserves the set of best responses but need not preserve the action selected from that set. The proof compares optimizers at two different contracts and takes a limit, so no shift-invariance property of the selector is used.

For $y\in\mathbb R^m$, we write $y_{-m}=(y_1,\ldots,y_{m-1})$.
For each contract $f$, define its payment-difference vector
relative to outcome $m$ by $x(f)=f_{-m}-f_m\mathbf1$.
Thus $x_i(f)=f_i-f_m$ is the payment for outcome $i$
minus the payment for the reference outcome $m$.
The range of $x(f)$ over contracts $f\in[0,1]^m$ is
\begin{equation}\label{eq:reduced-domain}
\mathcal D=\left\{x\in\mathbb R^{m-1}:
\max(0,x_1,\ldots,x_{m-1})
-\min(0,x_1,\ldots,x_{m-1})\le1\right\}.
\end{equation}
For $x\in\mathcal D$, let
$\lambda(x)=-\min(0,x_1,\ldots,x_{m-1})$ and define
\begin{equation}\label{eq:canonical-section}
    \bar f(x)=(x+\lambda(x)\mathbf1,\lambda(x)),\qquad
    \bar v=v_{-m}-v_m\mathbf1,\qquad
    Q(x)=\bigl(P(\bar f(x))\bigr)_{-m}.
\end{equation}
The shift $\lambda(x)$ subtracts the smallest entry of $(x,0)$
from every coordinate. By~\eqref{eq:reduced-domain}, the resulting
contract $\bar f(x)$ lies in $[0,1]^m$, has minimum payment zero,
and satisfies $x(\bar f(x))=x$. It is therefore the unique
normalized contract realizing the payment differences $x$.
The vector $\bar v$ expresses the principal's values relative
to the same reference outcome $m$, and $Q(x)$ records the first $m-1$ outcome
probabilities induced by $\bar f(x)$. The remaining probability
is determined by
$\bigl(P(\bar f(x))\bigr)_m=1-\mathbf1\cdot Q(x)$.

\begin{lemma}[Canonical reduction]\label{lem:reduction}
$\mathcal D$ is compact and convex.
{The map $f \mapsto x(f)$ restricts to a bijection from the normalized contracts onto $\mathcal D$, with inverse $\bar f(\cdot)$.}
Moreover, $\bar v\in\mathcal D$, $\lambda$ is $1$-Lipschitz, and for all $x,y\in\mathcal D$,
\begin{equation}\label{eq:reduced-monotonicity}
    (Q(x)-Q(y))\cdot(x-y)\ge0,
\end{equation}
while
\begin{equation}\label{eq:reduced-utility}
    u(\bar f(x))=v_m-\lambda(x)+Q(x)\cdot(\bar v-x).
\end{equation}
\end{lemma}

\begin{proof}
Equivalently, $\mathcal D=\{x:-1\le x_i\le1,\ x_i-x_j\le1\ \forall i,j\}$, so it is a compact convex polytope.
{For every $x\in\mathcal D$, the contract $\bar f(x)$ is normalized and $x(\bar f(x))=x$. Conversely, if $f$ is normalized, then $f_m=\lambda(x(f))$ and $\bar f(x(f))=f$.}
Because
$\bar f(x)-\bar f(y)=(x-y,0)+(\lambda(x)-\lambda(y))\mathbf1$ and outcome laws have unit mass, Lemma~\ref{lem:monotone} yields~\eqref{eq:reduced-monotonicity}. The coordinate range of $v$ is at most one, giving $\bar v\in\mathcal D$, and it holds that the minimum function is $1$-Lipschitz in $\ell_\infty$. Finally, substituting
{$\bigl(P(\bar f(x))\bigr)_m=1-\mathbf1\cdot Q(x)$}
gives~\eqref{eq:reduced-utility}.
\end{proof}

Define the reduced Minty coordinate $z(x)=x+Q(x)\in[-1,2]^{m-1}$ and residual $F_z(x)=x+Q(x)-z$.

\begin{lemma}[Reduced Minty geometry]\label{lem:minty}
For every $z$ and $x,y\in\mathcal D$,
$(F_z(x)-F_z(y))\cdot(x-y)\ge\|x-y\|_2^2$. Moreover,
\[
\|x-y\|_2\le\|z(x)-z(y)\|_2,
\qquad \|Q(x)-Q(y)\|_2\le\|z(x)-z(y)\|_2,
\]
and $|u(\bar f(x))-u(\bar f(y))|\le C_m\|z(x)-z(y)\|_2$.
\end{lemma}

\begin{proof}
The first claim follows from $F_z(x)-F_z(y)=(x-y)+(Q(x)-Q(y))$ and~\eqref{eq:reduced-monotonicity}. Expanding
$\|z(x)-z(y)\|_2^2=\|x-y\|_2^2+\|Q(x)-Q(y)\|_2^2+2(Q(x)-Q(y))\cdot(x-y)$ proves both distance inequalities and injectivity of $z(\cdot)$ on the selected-response graph. Finally,~\eqref{eq:reduced-utility}, Lipschitzness of $\lambda$, $\|\bar v-x\|_2\le2\sqrt{m-1}$, and $\|Q(y)\|_2\le1$ give
$|u(\bar f(x))-u(\bar f(y))|\le C_m(\|x-y\|_2+\|Q(x)-Q(y)\|_2)$ and applying the distance bounds finishes the proof.
\end{proof}

The lemma controls utility along the selected-response graph shown in Figure~\ref{fig:minty-geometry}, not as a continuous function of the posted contract. The next section learns this graph without observing it.

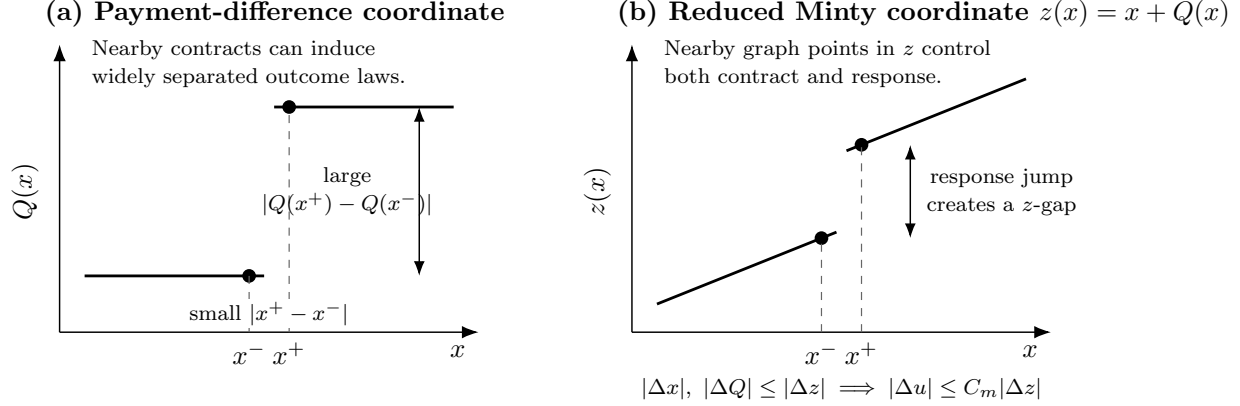
\begin{figure}[t]
  \centering
  \resizebox{\linewidth}{!}{%

\begin{tikzpicture}[
    x=0.82cm,
    y=0.82cm,
    >=Latex,
    axis/.style={black, line width=0.55pt, -{Latex[length=2.2mm]}},
    graph/.style={black, line width=1.05pt},
    guide/.style={black!60, line width=0.45pt, dashed},
    measure/.style={black, line width=0.55pt, <->},
    point/.style={circle, fill=black, inner sep=1.8pt},
    every node/.style={font=\small},
    note/.style={font=\scriptsize, align=center}
]

\begin{scope}
  \node[anchor=west,font=\small\bfseries] at (0.25,6.05)
    {(a) Payment-difference coordinate};

  \draw[axis] (0.65,0.95) -- (7.35,0.95) node[below left=1pt] {$x$};
  \draw[axis] (0.65,0.95) -- (0.65,5.55);
  \node[rotate=90] at (0.10,3.20) {$Q(x)$};

  \draw[graph] (1.05,1.85) -- (3.92,1.85);
  \draw[graph] (4.08,4.55) -- (6.95,4.55);

  \coordinate (qm) at (3.68,1.85);
  \coordinate (qp) at (4.32,4.55);
  \node[point] at (qm) {};
  \node[point] at (qp) {};
  \draw[guide] (3.68,0.95) -- (qm);
  \draw[guide] (4.32,0.95) -- (qp);
  \node[below] at (3.68,0.95) {$x^{-}$};
  \node[below] at (4.32,0.95) {$x^{+}$};

  \draw[measure] (3.68,1.25) -- (4.32,1.25);
  \node[note,fill=white,inner sep=1pt] at (4.00,1.25)
    {small $|x^{+}-x^{-}|$};

  \draw[measure] (6.40,1.85) -- (6.40,4.55);
  \node[note] at (5.25,3.20)
    {large\\$|Q(x^{+})-Q(x^{-})|$};

  \node[note,align=left,anchor=west] at (1.00,5.25)
    {Nearby contracts can induce\\widely separated outcome laws.};
\end{scope}

\begin{scope}[shift={(9.15,0)}]
  \node[anchor=west,font=\small\bfseries] at (0.25,6.05)
    {(b) Reduced Minty coordinate $z(x)=x+Q(x)$};

  \draw[axis] (0.65,0.95) -- (7.35,0.95) node[below left=1pt] {$x$};
  \draw[axis] (0.65,0.95) -- (0.65,5.55);
  \node[rotate=90] at (0.10,3.20) {$z(x)$};

  \draw[graph] (1.05,1.40) -- (3.92,2.55);
  \draw[graph] (4.08,3.85) -- (6.95,5.00);

  \coordinate (zm) at (3.68,2.454);
  \coordinate (zp) at (4.32,3.946);
  \node[point] at (zm) {};
  \node[point] at (zp) {};
  \draw[guide] (3.68,0.95) -- (zm);
  \draw[guide] (4.32,0.95) -- (zp);
  \node[below] at (3.68,0.95) {$x^{-}$};
  \node[below] at (4.32,0.95) {$x^{+}$};

  \draw[measure] (5.10,2.454) -- (5.10,3.946);
  \node[note,anchor=west] at (5.20,3.20)
    {response jump\\creates a $z$-gap};

  \node[note,align=left,anchor=west] at (1.00,5.25)
    {Nearby graph points in $z$ control\\both contract and response.};

  \node[note] at (4.00,0.05)
    {$|\Delta x|,\ |\Delta Q|\le |\Delta z|
      \;\Longrightarrow\; |\Delta u|\le C_m|\Delta z|$};
\end{scope}

\end{tikzpicture}
  }
\caption{Binary illustration of the reduced selected-response graph and its
Minty coordinate $z=x+Q(x)$.}
  \label{fig:minty-geometry}
\end{figure}

\section{The Minty-Coordinate Algorithm}\label{sec:algorithm}

Lemma~\ref{lem:minty} shows that proximity in Minty coordinates
controls utility differences along the selected-response graph.
However, the outcome response map $P$, and hence the reduced
map $Q$, is unknown. The learner therefore cannot directly
determine which Minty coordinates are attained or recover
the corresponding contracts.

The algorithm has two levels. At the first level, a simulator
posts contracts for a fixed candidate Minty coordinate $z$
and updates its payment differences using observed outcomes.
If the Minty coordinate of a near-optimal normalized
contract were known, a single simulator would suffice.
To search for such a coordinate, the algorithm discretizes
the known cube $[-1,2]^{m-1}$ and maintains one simulator
for each grid point. At the second level, the master allocates rounds
among these simulators, selecting one in each round based on
the rewards observed so far.

We fix $\varepsilon\in(0,1)$, partition each coordinate interval
$[-1,2]$ into $\lceil3\sqrt{m-1}/\varepsilon\rceil$ equal cells,
and let $\mathcal Z_\varepsilon$ be the Cartesian product of
their centers. Every point in $[-1,2]^{m-1}$ is within
distance $\varepsilon/2$ of this grid, and
\[
    K_\varepsilon:=|\mathcal Z_\varepsilon|
    \le
    \left(1+\frac{3\sqrt{m-1}}{\varepsilon}\right)^{m-1}.
\]
The grid can be constructed without knowing $Q$. It
approximates every attained Minty coordinate, although some
grid points may not equal $x+Q(x)$ for any $x\in\mathcal D$.

Fix a grid point $z\in\mathcal Z_\varepsilon$ and consider its
simulator. The superscript $z$ identifies the simulator, and
the subscript $n$ indexes its calls. Let $x_n^z\in\mathcal D$
be its payment-difference vector used on its $n$-th call.
Each simulator starts at $x_1^z=0$, corresponding to the
feasible zero-payment contract $\bar f(0)=0$.
On call $n$, it posts the normalized contract $\bar f(x_n^z)$
defined in~\eqref{eq:canonical-section} and records the observed
outcome as the one-hot vector $X_n^z$.
We denote the first $m-1$ entries of $X_n^z$ as $X_{n,-m}^z$.
By~\eqref{eq:conditional-mean} and the definition of $Q$
in~\eqref{eq:canonical-section}, this vector has conditional
mean $Q(x_n^z)$. Hence,
\[
    \mathbb E[x_n^z+X_{n,-m}^z-z\mid x_n^z]
    =x_n^z+Q(x_n^z)-z
    =F_z(x_n^z).
\]
When $z$ is attained, strong monotonicity
(Lemma~\ref{lem:minty}) makes the negative residual point
toward its corresponding payment-difference vector.
Algorithm~\ref{alg:simulator} subtracts this residual estimate,
scaled by $1/(n+1)$, from $x_n^z$. It then applies the Euclidean
projection $\Pi_{\mathcal D}$ onto $\mathcal D$ to keep the next
payment-difference vector feasible.

\begin{algorithm}[t]
\caption{Minty-coordinate simulator for grid point $z$}\label{alg:simulator}
\textbf{Input:} public value vector $v$ and grid point $z\in\mathcal{Z}_\varepsilon$.\\
\textbf{Initialization:} $x_1^z=0\in\mathcal D$.\\
\textbf{State:} internal call count $n$ and current difference coordinate $x_n^z$.\\
\textbf{Output on each call:} posted contract, realized reward, and updated simulator state.\\
On the $n$-th call to simulator $z$:
\begin{enumerate}[leftmargin=2em]
\item Post the canonical contract $\bar f(x_n^z)$.
\item Observe one-hot outcome $X_n^z\in\{0,1\}^m$ with $\sum_j X_{n,j}^z=1$.
\item Receive reward $Y_n^z=X_n^z\cdot (v-\bar f(x_n^z))$.
\item Update
\[
    x_{n+1}^z
    =\Pi_{\mathcal D}\left(x_n^z-\frac{1}{n+1}(x_n^z+X_{n,-m}^z-z)\right).
\]
\end{enumerate}
\end{algorithm}

The master in Algorithm~\ref{alg:master} selects a simulator
using an upper confidence bound (UCB) index. For a simulator
that has been called $n_z\ge1$ times, the index is its average
observed reward plus a bonus. An uncalled simulator is assigned
index $+\infty$. In each global round, the master
calls the simulator with the largest index. The simulators
are rested, so only the selected simulator advances its call
count, updates its contract, and records a new reward.
All other simulators remain paused.

A simulator's expected reward can change as its contract is
updated. For a grid point close to a normalized contract's
Minty coordinate, Section~\ref{sec:proof-main} bounds the
simulator's average utility loss relative to that contract
with high probability. The bound consists of a grid
approximation error and a term that decreases with the number
of calls. The bonus covers the decreasing term and sampling noise.

\begin{algorithm}[t]
\caption{Rested-UCB master}\label{alg:master}
\textbf{Input:} value vector $v$, scale $\varepsilon$, grid $\mathcal{Z}_\varepsilon$, horizon $T$, confidence level $\delta\in(0,1)$, and the fixed dimension-only design constant $B_m$.\\
\textbf{Output:} the sequence of $T$ posted contracts generated by selected simulators.\\
Let $\iota=1+\log^3(16K_\varepsilon T/\delta)$.
Initialize $n_z=0$ for every $z\in\mathcal{Z}_\varepsilon$.\\
For rounds $t=1,\ldots,T$:
\begin{enumerate}[leftmargin=2em]
\item For each $z\in\mathcal{Z}_\varepsilon$, if $n_z=0$ set $I_z=+\infty$. If $n_z\ge1$, set
\[
\begin{aligned}
        \widehat\mu_z(n_z)&=\frac{1}{n_z}\sum_{k=1}^{n_z}Y_k^z,\qquad
        I_z=\widehat\mu_z(n_z)+B_m\frac{\iota}{\sqrt{n_z}}.
\end{aligned}
\]
\item Let $z_t$ be the lexicographically first element of $\argmax_{z\in\mathcal{Z}_\varepsilon} I_z$.
\item Call simulator $z_t$ once according to Algorithm~\ref{alg:simulator}, using internal call index $n_{z_t}+1$.
\item Set $n_{z_t}\leftarrow n_{z_t}+1$, leaving all other simulator states unchanged.
\end{enumerate}
\end{algorithm}

A direct implementation stores $O(mK_\varepsilon)$ numbers. Projection onto $\mathcal D$ is a convex quadratic program over a known fixed-dimensional polytope and requires no instance-specific information. The master can select the largest index by scanning the $K_\varepsilon$ simulators.

\section{Proof of Theorem~\ref{thm:main}}\label{sec:proof-main}

We prove the theorem using two lemmas.
Lemma~\ref{lem:prefix} bounds a simulator's cumulative utility
loss relative to a normalized contract in terms of the distance
between their Minty coordinates.
Lemma~\ref{lem:ucb-accounting} combines this bound with reward
concentration to control the master's total regret.
We apply these lemmas to a near-optimal normalized contract
and then choose the grid resolution.

\begin{lemma}[Prefix regret of a Minty simulator]\label{lem:prefix}
Fix a candidate Minty coordinate $z\in[-1,2]^{m-1}$
and a comparator $x^\star\in\mathcal D$.
Let $z^\star=x^\star+Q(x^\star)$.
Run the update in Algorithm~\ref{alg:simulator} with this
candidate $z$, and write $x_n$ for $x_n^z$, the
payment-difference vector used on its $n$-th call.
For every $T\ge2$ and $\delta\in(0,1)$, with probability
at least $1-\delta$,
\[
\begin{aligned}
    \sum_{n=1}^N\big(u(\bar f(x^\star))-u(\bar f(x_n))\big)
    &\le C_m\left(N\|z-z^\star\|_2+\sqrt N\,\iota\right)
\end{aligned}
\]
simultaneously for all $1\le N\le T$, where $\iota=1+\log^3(8T/\delta)$.
\end{lemma}

The bound separates the error induced by approximating $z^\star$
with $z$ from the error accumulated while updating the contract.

For the analysis, we generate a sequence of $T$ calls for each
simulator, with outcomes drawn according to its posted
contracts. Whenever the master selects a simulator, it
reveals the next outcome in that sequence.
Since unselected simulators remain paused, this construction
has the same distribution as the original interaction.
The bounds below therefore apply to every prefix, including
those that the master does not observe.

\begin{lemma}[Rested-UCB accounting for the master]\label{lem:ucb-accounting}
Run Algorithm~\ref{alg:master} on a finite grid $\mathcal Z$ of size $K$ for $T$ rounds. Fix a comparator $x^\star\in\mathcal D$, an approximation scale $\varepsilon\ge0$, and an error scale $\iota\ge1$. Suppose that, on some event, the following two certificates hold.
\begin{enumerate}
    \item There is a special simulator $z^\circ$ such that, for every $1\le n\le T$,
    \[
        \frac1n\sum_{k=1}^n u(\bar f(x_k^{z^\circ}))
        \ge u(\bar f(x^\star))-C_m\left(\varepsilon+\frac{\iota}{\sqrt n}\right).
    \]
    \item For every simulator $z$ and every $1\le n\le T$,
    \[
        \left|\widehat\mu_z(n)-\frac1n\sum_{k=1}^n u(\bar f(x_k^z))\right|
        \le C_m\frac{\iota}{\sqrt n}.
    \]
\end{enumerate}
Let $\beta(n)$ be the actual bonus in the index $I_z(n)=\widehat\mu_z(n)+\beta(n)$ for a simulator with $n\ge1$ previous calls. Suppose the design constant in Algorithm~\ref{alg:master} is chosen so that, for every $1\le n\le T$, this bonus dominates the sum of the two $n^{-1/2}$ error terms in the certificates above and also satisfies $\beta(n)\le C_m\iota/\sqrt n$. Then, on the same event,
\[
    \sum_{t=1}^T\big(u(\bar f(x^\star))-u(f_t)\big)
    \le C_m\left(T\varepsilon+\iota\sqrt{KT}+K\right).
\]
\end{lemma}

The complete proofs of Lemmas~\ref{lem:prefix} and~\ref{lem:ucb-accounting} are in Appendices~\ref{app:simulator} and~\ref{app:master}. We first combine their conclusions to prove the theorem.

\begin{proof}[Proof of Theorem~\ref{thm:main}]
Fix an arbitrary instance $I\in\mathcal I_m$ and write $u$ for its utility map. Fix $\varepsilon\in(0,1)$ and run Algorithms~\ref{alg:simulator} and~\ref{alg:master} with confidence parameter $\delta\in(0,1)$, which will be set to $T^{-2}$ at the end. Let $\iota=1+\log^3(16K_\varepsilon T/\delta)$.
This factor dominates the simulator factor in Lemma~\ref{lem:prefix} when that lemma is invoked with confidence $\delta/2$, since we have  $K_\varepsilon\ge1$.

Because $u$ may be discontinuous, the supremum need not be attained. For any fixed $\alpha>0$, Lemmas~\ref{lem:normalization} and~\ref{lem:reduction} allow us to choose an $\alpha$-optimal comparator $x^\star\in\mathcal D$ satisfying
$u(\bar f(x^\star))\ge \sup_{f\in[0,1]^m}u(f)-\alpha$.
Let $z^\star=x^\star+Q(x^\star)$.
Since $\mathcal{Z}_\varepsilon$ is an $\varepsilon$-net of $[-1,2]^{m-1}$, we can choose $z^\circ\in\mathcal{Z}_\varepsilon$ such that
$\|z^\circ-z^\star\|_2\le\varepsilon$.
Lemma~\ref{lem:prefix}, applied to simulator $z^\circ$ with confidence $\delta/2$, gives an event of probability at least $1-\delta/2$ on which, for all $1\le n\le T$,
\begin{equation}\label{eq:special-prefix-main}
    \frac1n\sum_{k=1}^n u(\bar f(x_k^{z^\circ}))
    \ge u(\bar f(x^\star))-C_m\left(\varepsilon+\frac{\iota}{\sqrt n}\right).
\end{equation}
Lemma~\ref{lem:reward-concentration} gives the uniform reward concentration event that with probability at least $1-\delta/2$, for all $z\in\mathcal{Z}_\varepsilon$ and all $1\le n\le T$,
\begin{equation}\label{eq:reward-concentration-main}
    \left|\widehat\mu_z(n)-\frac1n\sum_{k=1}^n u(\bar f(x_k^z))\right|
    \le C\sqrt{\frac{\log(16K_\varepsilon T/\delta)}{n}}.
\end{equation}

On the intersection of~\eqref{eq:special-prefix-main} and~\eqref{eq:reward-concentration-main}, the hypotheses of Lemma~\ref{lem:ucb-accounting} hold with $K=K_\varepsilon$ and error scale $\iota$. Indeed, we know $\iota\ge1$ and it dominates $\sqrt{\log(16K_\varepsilon T/\delta)}$. By the fixed choice of the dimension-only constant $B_m$, the actual bonus $\beta(n)$ dominates the sum of the two $n^{-1/2}$ error terms while remaining at most $C_m\iota/\sqrt n$.
Therefore, at this intersection,

\[
\begin{aligned}
    T\sup_{f\in[0,1]^m}u(f)-\sum_{t=1}^T u(f_t)
    \le T\alpha+Tu(\bar f(x^\star))-\sum_{t=1}^T u(f_t)
    \le T\alpha+C_m
      \left(T\varepsilon+\iota\sqrt{T K_\varepsilon}+K_\varepsilon\right).
\end{aligned}
\]

On the complement of the high-probability event, the per-round regret relative to the optimal one-period value is at most $2$ because $u(f)\in[-1,1]$. Hence
\[
\begin{aligned}
    T\sup_{f\in[0,1]^m}u(f)-\mathbb{E}\left[\sum_{t=1}^T u(f_t)\right]
    &\le C_m\left(T\varepsilon+\iota\sqrt{T K_\varepsilon}+K_\varepsilon\right)
      +2T\delta+T\alpha.
\end{aligned}
\]
Setting $\delta=T^{-2}$ yields $\iota=1+\log^3(16K_\varepsilon T^3)=L_T$, and $2T\delta=2/T\le1\le K_\varepsilon$ for $T\ge2$. Letting $\alpha\downarrow0$ proves the fixed-scale bound for $\mathfrak R_T$. Since $I\in\mathcal I_m$ was arbitrary, the same bound holds for every instance in the class as well.

For the optimized rate, we substitute the bound on $K_\varepsilon$. In fixed-$m$ notation, we have
\[
\begin{aligned}
    T\varepsilon+L_T\sqrt{TK_\varepsilon}+K_\varepsilon
    &=\widetilde{O}_m\big(T\varepsilon+T^{1/2}\varepsilon^{-(m-1)/2}
      +\varepsilon^{-(m-1)}\big).
\end{aligned}
\]
Taking $\varepsilon\asymp T^{-1/(m+1)}$ balances the first two terms at $T^{m/(m+1)}$, while the third term is $T^{(m-1)/(m+1)}$ and has lower order. This proves the optimized rate.
\end{proof}

\section{A Matching Lower Bound}\label{sec:lower}

The lower bound uses one type, the public value vector $v=(1,\ldots,1,0)$, and a finite product family of actions. For a contract $f$, write $x_i=f_i-f_m$ for $i<m$. At scalar level $k$, we choose a non-null outcome probability $\rho_k$ and a cost contribution $\kappa_k$ so that
\[
    \phi_k(x)=\rho_kx-\kappa_k
\]
and adjacent lines cross at $k\varepsilon$. A joint action selects one level in each of the $m-1$ non-null coordinates, making the baseline agent payoff a sum of scalar envelopes. The probabilities are calibrated so that the principal's baseline utility is at most $1/2$ under every contract.

For each even multi-index $l$ in a packing set $\mathcal L$, the alternative $I_l$ lowers only the cost of joint action $a_l$ by
\[
    \Delta=\frac{\varepsilon^2}{16(m-1)}.
\]
The action $a_l$ can be selected only in the aggregate response region
\begin{equation}\label{eq:lower-region-overview}
    \mathcal R_l=
    \left\{x:
    \sum_{i=1}^{m-1}
    \left[\max_k\phi_k(x_i)-\phi_{l_i}(x_i)\right]\le\Delta
    \right\}.
\end{equation}
This region need not be a Cartesian product. Neighbor comparisons place it inside a narrow coordinate box, so even multi-indices make the regions pairwise disjoint. A small displacement in one coordinate gives a contract at which $a_l$ is the unique best response and the principal gains a constant multiple of $\varepsilon$ over the baseline ceiling. Outside $\mathcal R_l$, the alternative and baseline responses coincide.

\paragraph{Relation to the earlier lower-bound construction.}
Theorem~5 of \citet{zhu2023sample} denotes the total number of outcomes by $m$, whereas its Appendix~D indexes an action by $m$ non-null coordinates and then adds a null-outcome coordinate to the production vector. Under the total-outcome convention used here, the construction therefore has $m-1$ non-null coordinates. This bookkeeping issue alone can be repaired by relabeling the dimension, but the displayed Cartesian response cell has a separate multidimensional problem. At its simultaneous lower corner, each non-null coordinate incurs an unperturbed agent-payoff loss equal to the full cost reduction assigned to the distinguished action. Those losses sum to $m-1$ times the reduction, while the distinguished action receives the reduction only once. Relative to the action obtained by lowering every coordinate by one level, its payoff difference is therefore $-(m-2)$ times the reduction. For $m>2$, the difference is strictly negative, so the displayed corner does not induce the distinguished action and the associated principal-payoff calculation does not follow. The binary case $m=2$ is unaffected. Our proof replaces the Cartesian characterization with the aggregate condition~\eqref{eq:lower-region-overview} and no product response cell is needed.

\begingroup
\emergencystretch=2em
A query in $\mathcal R_l$ contributes $O(\varepsilon^2)$ to the one-period baseline-to-alternative KL divergence, whereas a query outside $\mathcal R_l$ contributes zero. Region disjointness bounds the total baseline visit count across alternatives by $T$. A uniformly sampled round and the Bretagnolle-Huber inequality then reduce the problem to testing among $|\mathcal L|\asymp\varepsilon^{-(m-1)}$ alternatives. Averaging the resulting KL bounds over $l$ gives $O_m(T\varepsilon^{m+1})$ average divergence. Choosing $\varepsilon\asymp T^{-1/(m+1)}$ keeps the testing error bounded away from zero and yields regret $\Omega_m(T\varepsilon)$. Appendix~\ref{app:lower} details the construction, response regions, profitable contract, and adaptive information bounds.
\par\endgroup

\section{Conclusion}

We characterize the minimax difficulty of learning an unrestricted bounded contract from repeated outcome feedback. With $m$ total outcomes, the exact fixed-dimensional regret exponent is $m/(m+1)$ up to logarithmic factors. The upper bound allows arbitrary action spaces, heterogeneity, and discontinuous selected responses, while the lower bound already holds for one type, finitely many actions, bounded nonnegative costs, and a common deterministic priority rule. Thus the rate is not driven by a rich hard instance class.

The raw payment vector overstates the dimension of the incentive problem by one, and we find that the benchmark can be normalized even when the fixed selection among best responses is not invariant to common payment shifts. In payment-difference coordinates, revealed preference supplies enough monotone structure to obtain the sharp upper rate without assuming continuity. A matching family of aggregate response regions shows that this dependence on the number of outcomes is unavoidable.

The result separates three modeling choices that are sometimes bundled together. Unrestricted contracts do not by themselves make learning impossible, outcome categories provide the vector feedback needed to exploit revealed preference, and a response selection fixed before learning makes the induced environment stationary. Relaxing the feedback or stationarity conditions leads to different problems, while adding behavioral structure may permit faster learning. A particularly natural next step is to combine the present learning cost with the economic benefit of refining or coarsening the set of contractible performance categories.

\appendix
\section{Concentration Lemmas}

\begin{lemma}[Uniform Freedman bound]\label{lem:freedman}
Let $(Z_n)_{n=1}^T$ be martingale differences with $|Z_n|\le b$, and put
$V_N=\sum_{n=1}^N\mathbb E[Z_n^2\mid\mathcal F_{n-1}]$ and
$\iota=1+\log^3(8T/\delta)$. With probability at least $1-\delta$, simultaneously for all $N\le T$, it holds that
\[
    \sum_{n=1}^N Z_n\le C\big(\sqrt{\iota V_N}+b\iota\big).
\]
\end{lemma}

\begingroup
\emergencystretch=2em
\begin{proof}
Freedman's maximal inequality \citep[Theorem~1.6]{freedman1975tail} bounds
$\mathbb P(\exists N\le T:\sum_{n\le N}Z_n\ge s,\ V_N\le v)$ by
$\exp\{-s^2/[2(v+bs/3)]\}$. We apply it to the dyadic variance ranges
$V_N\le b^2$ and $b^22^{j-1}<V_N\le b^22^j$ and assign failure probability
$6\delta/[\pi^2(j+1)^2]$ to range $j$. There are at most
$1+\lceil\log_2T\rceil$ nonempty ranges, and their Freedman thresholds are bounded by
$C(\sqrt{\iota V_N}+b\iota)$. A union bound proves the claim.
\end{proof}
\endgroup

For adaptive allocation, we can use the standard call-time coupling that independently seeds every simulator-call pair, reveals a seed only when that simulator is selected, and extends each virtual stream to $T$ calls. This leaves the observed interaction unchanged and makes every simulator prefix available in its own filtration.

\begin{lemma}[Uniform reward concentration]\label{lem:reward-concentration}
For any adaptive master over a grid $\mathcal Z$ of size $K$, with probability at least $1-\delta$, simultaneously for all $z\in\mathcal Z$ and $n\le T$, it holds that
\[
\left|\widehat\mu_z(n)-\frac1n\sum_{k=1}^n u(\bar f(x_k^z))\right|
\le C\sqrt{\frac{\log(8KT/\delta)}{n}}.
\]
\end{lemma}

\begin{proof}
In simulator call time, $\bar f(x_k^z)$ is predictable and
$Y_k^z=X_k^z\cdot(v-\bar f(x_k^z))\in[-1,1]$ has conditional mean
$u(\bar f(x_k^z))$. Hoeffding-Azuma~\citep[Exercise~20.6]{lattimore2020bandit} for each $(z,n)$, followed by a union bound over the $KT$ pairs, proves the claim. Because the event holds for all virtual prefixes, restricting it to prefixes revealed by the adaptive master needs no optional-stopping argument.
\end{proof}

\section{Proof of the Simulator Prefix Bound}\label{app:simulator}

The proof has two logically distinct stages. First, strong monotonicity of the Minty residual yields pathwise tracking of the comparator up to the grid mismatch $\Delta$. Second, a projection inequality against the value coordinate $\bar v$ converts that tracking statement into utility regret. The second stage is necessary because contract distance by itself cannot control utility across a discontinuous response boundary.

\begin{proof}[Proof of Lemma~\ref{lem:prefix}]
Let's suppress the simulator superscript and write $q_n=Q(x_n)$,
$q^\star=Q(x^\star)$, $z^\star=x^\star+q^\star$,
$\Delta=\|z-z^\star\|_2$, and
$S_N=\sum_{n=1}^N\|x_n-x^\star\|_2^2$. Since
$F_z(x_n)=(x_n-x^\star)+(q_n-q^\star)-(z-z^\star)$,
monotonicity~\eqref{eq:reduced-monotonicity} and Young's inequality~\citep{young1912classes} give
\begin{equation}\label{eq:tracking-core}
F_z(x_n)\cdot(x_n-x^\star)
\ge\tfrac12\|x_n-x^\star\|_2^2-\tfrac12\Delta^2.
\end{equation}

Let $\xi_n=X_{n,-m}-q_n$. In simulator call time, $\xi_n$ is a bounded martingale difference, and the update direction is bounded by $C_m$. With $\eta_n=1/(n+1)$, projection nonexpansiveness and~\eqref{eq:tracking-core} yield
\begin{equation}\label{eq:tracking-recursion}
\|x_{n+1}-x^\star\|_2^2
\le(1-\eta_n)\|x_n-x^\star\|_2^2+\eta_n\Delta^2+C_m\eta_n^2
-2\eta_n\xi_n\cdot(x_n-x^\star).
\end{equation}
The last scalar increment is bounded by $C_m$ and has conditional variance at most
$C_m\|x_n-x^\star\|_2^2$. Multiplying~\eqref{eq:tracking-recursion} by $n+1$, summing, and applying Lemma~\ref{lem:freedman} with confidence $\delta/2$ therefore gives, simultaneously for $N\le T$,
\begin{equation}\label{eq:tracking-bootstrap}
\|x_{N+1}-x^\star\|_2^2
\le\Delta^2+\frac{C_m\iota}{N+1}
+\frac{C_m\sqrt{\iota S_N}}{N+1},
\qquad \iota=1+\log^3(8T/\delta).
\end{equation}
Indeed, the telescoped noise sum is at most
$C_m(\sqrt{\iota S_N}+\iota)$ as the harmonic deterministic term is absorbed by
$C_m\iota$. Replacing $\delta$ by $\delta/2$ changes $\iota$ only by a universal factor, absorbed into $C_m$.

The variance is self-normalizing. As the simulator approaches the comparator, the martingale term becomes correspondingly smaller. \eqref{eq:tracking-bootstrap} therefore closes through a bootstrap rather than the cruder $O(\sqrt N)$ noise bound, which would lose the required uniform prefix rate.

Summing~\eqref{eq:tracking-bootstrap} for $N<r$ and using monotonicity of $S_N$ gives
$S_r\le C_m+r\Delta^2+C_m\iota\log(r+1)+
C_m\sqrt{\iota S_r}\log(r+1)$. Solving this quadratic inequality and using
$\log^2(r+1)\le C\iota$ yields
\begin{equation}\label{eq:tracking-consequences}
S_r\le C_m(r\Delta^2+\iota^2),\qquad
\sum_{n=1}^N\|x_n-x^\star\|_2
\le C_m(N\Delta+\sqrt N\,\iota).
\end{equation}
Substituting back into~\eqref{eq:tracking-bootstrap}, with Young's inequality for the cross term, also gives
$\|x_{N+1}-x^\star\|_2\le C_m(\Delta+\iota/\sqrt{N+1})$.

It remains to convert tracking into utility. Substituting
$q^\star-q_n=-F_z(x_n)+(x_n-x^\star)-(z-z^\star)$ into
the reduced utility identity~\eqref{eq:reduced-utility}, and then using
~\eqref{eq:tracking-core}, Lipschitzness of $\lambda$, and boundedness of
$\mathcal D$ and the Minty cube (so $\Delta^2\le C_m\Delta$), gives
\begin{equation}\label{eq:utility-conversion}
u(\bar f(x^\star))-u(\bar f(x_n))
\le F_z(x_n)\cdot(x_n-\bar v)
+C_m\|x_n-x^\star\|_2+C_m\Delta.
\end{equation}
This is the only point at which utility enters the tracking argument.

The residual term in~\eqref{eq:utility-conversion} is not bounded pointwise. Instead, we test the same projected update against $\bar v$, which is feasible by Lemma~\ref{lem:reduction}. Its prefix sum then telescopes with the increasing weights generated by $\eta_n^{-1}=n+1$.

Because $\bar v\in\mathcal D$, projection nonexpansiveness also gives, with
$a_n=\|x_n-\bar v\|_2^2$,
\[
F_z(x_n)\cdot(x_n-\bar v)
\le\frac{a_n-a_{n+1}}{2\eta_n}+C_m\eta_n-\xi_n\cdot(x_n-\bar v).
\]
The last terms are bounded martingale differences, so Lemma~\ref{lem:freedman} with confidence $\delta/2$ bounds their prefixes by $C_m\sqrt N\,\iota$. For the weighted telescope, we set $a_\star=\|x^\star-\bar v\|_2^2$. Summation by parts gives
\[
\sum_{n=1}^N(n+1)(a_n-a_{n+1})
=2(a_1-a_\star)-(N+1)(a_{N+1}-a_\star)
+\sum_{n=2}^N(a_n-a_\star).
\]
Since $|a_n-a_\star|\le C_m\|x_n-x^\star\|_2$, the pointwise bound above and~\eqref{eq:tracking-consequences} show that the absolute contribution of the right-hand side is at most
$C_m(N\Delta+\sqrt N\,\iota)$. Hence, using
$\sum_{n\le N}\eta_n\le C\log(N+1)\le C\sqrt N\,\iota$, we know that
\[
\sum_{n=1}^NF_z(x_n)\cdot(x_n-\bar v)
\le C_m(N\Delta+\sqrt N\,\iota).
\]
Summing~\eqref{eq:utility-conversion} and applying~\eqref{eq:tracking-consequences} proves the claimed prefix bound for every $N\le T$. The two concentration events have joint probability at least $1-\delta$.
\end{proof}

\section{Proof of the Rested-UCB Accounting Lemma}\label{app:master}

\begin{proof}[Proof of Lemma~\ref{lem:ucb-accounting}]
Write $I_z(n)=\widehat\mu_z(n)+\beta(n)$ for $n\ge1$ and
$I_z(0)=+\infty$. The two certificates and the lower bound on $\beta(n)$ imply
$I_{z^\circ}(n)\ge u(\bar f(x^\star))-C_m\varepsilon$ for every $n$, including $n=0$. Thus every simulator selected by the master has index at least this value.

Let $N_z$ be the final number of calls to $z$. The cases with $N_z\le1$ cost at most two and are absorbed by the final $K$ term. If $N_z\ge2$, inspect the index just before its last call. At prefix $N_z-1$, index optimality, reward concentration, and
$\beta(n)\le C_m\iota/\sqrt n$ give
\[
\frac1{N_z-1}\sum_{k=1}^{N_z-1}u(\bar f(x_k^z))
\ge u(\bar f(x^\star))-C_m\varepsilon-
C_m\frac{\iota}{\sqrt{N_z-1}}.
\]
The final call costs at most two because utilities lie in $[-1,1]$. Hence simulator $z$ contributes at most
$C_mN_z\varepsilon+C_m\iota\sqrt{N_z}+C$ pseudo-regret. Summing over $z$, using $\sum_zN_z=T$ and Cauchy-Schwarz inequality,
$\sum_z\sqrt{N_z}\le\sqrt{KT}$, gives
$C_m(T\varepsilon+\iota\sqrt{KT}+K)$. The rested prefixes partition the global rounds and prove the lemma.
\end{proof}
For unit-bounded contracts, our upper-bound theorem extends the outcome-feedback model of \citet{bacchiocchi2025learning} from finite-action, single-type environments to a substantially broader
class. For a constant number of actions, they obtain a high-probability regret bound of $\widetilde O(T^{4/5})$. We remove the finite-action and principal-favoring tie-breaking restrictions, permit fresh heterogeneous types and arbitrary, possibly infinite, action sets, and use a policy that neither enumerates nor identifies actions. For this broader class, our theorems guarantee a sharp expected regret $\widetilde\Theta_m(T^{m/(m+1)})$ under any
fixed deterministic selected-response convention.

\section{Proof of the Matching Lower Bound}\label{app:lower}

Let's fix $m\ge2$ and $0<\varepsilon\le1/64$, and choose $v=(1,\ldots,1,0)$,
$J=\lfloor(4\varepsilon)^{-1}\rfloor$, and
$\rho_k=[2(m-1)(1-k\varepsilon)]^{-1}$ for $0\le k\le J$. We also set
$\kappa_0=0$ and
$\kappa_k=\sum_{s=1}^ks\varepsilon(\rho_s-\rho_{s-1})$.
For every multi-index $k\in\{0,\ldots,J\}^{m-1}$, we can introduce an action $a_k$ with
$p_i(k)=\rho_{k_i}$ for $i<m$,
$p_m(k)=1-\sum_{i<m}\rho_{k_i}$, and
$c(k)=\sum_{i<m}\kappa_{k_i}$. A null action produces outcome $m$ at zero cost. We then fix one deterministic priority order for this common action set and all instances.

For a contract $f$, write $x_i=f_i-f_m$ and
$\phi_k(x)=\rho_kx-\kappa_k$. The agent utilities of $a_k$ and the null action are
$f_m+\sum_{i<m}\phi_{k_i}(x_i)$ and $f_m$, respectively.

\begin{lemma}[Baseline instance]\label{lem:lower-baseline}
The baseline is a valid finite-action instance with nonnegative costs bounded by one. Adjacent lines $\phi_{k-1}$ and $\phi_k$ cross at $k\varepsilon$ and a baseline best response using level $k_i\ge1$ satisfies $x_i\ge k_i\varepsilon$. Its principal utility obeys $\sup_fu_0(f)\le1/2$.
\end{lemma}

\begin{proof}
Because $J\varepsilon\le1/4$,
$\rho_0=1/[2(m-1)]\le\rho_k\le2/[3(m-1)]$, it holds that $p_m(k)\ge1/3$.
We also have $0\le\kappa_k\le k\varepsilon(\rho_k-\rho_0)\le k\varepsilon\rho_k
\le1/[6(m-1)]$, whence $c(k)\le1/6$.
The identity
$\phi_k(x)-\phi_{k-1}(x)
=(\rho_k-\rho_{k-1})(x-k\varepsilon)$
gives the ordered crossings, so level $k$ maximizes the scalar envelope
only on $[k\varepsilon,(k+1)\varepsilon]$, with the natural boundary
modifications. Because the action set contains every multi-index and
agent utility is additive across coordinates, any non-null best response
$a_k$ must satisfy
$k_i\in\argmax_{0\le s\le J}\phi_s(x_i)$ for every $i<m$. Otherwise,
replacing only coordinate $k_i$ by a better level would strictly increase the utility. Hence $k_i\ge1$ implies $x_i\ge k_i\varepsilon$, proving the stated coordinate condition.

If the null action is selected, principal utility is $-f_m\le0$. Otherwise, for the selected $a_k$, it holds that $u_0(f)=\sum_{i<m}\rho_{k_i}(1-x_i)-f_m$.
When $k_i\ge1$, the coordinate condition and
$\rho_{k_i}(1-k_i\varepsilon)=\rho_0$ give
$\rho_{k_i}(1-x_i)\le\rho_0$ and when $k_i=0$, feasibility gives
$\rho_0(1-x_i)\le\rho_0(1+f_m)$. Since $(m-1)\rho_0=1/2$, we obtain
$u_0(f)\le1/2-f_m/2\le1/2$.
\end{proof}

Let
$\mathcal L=\{2,4,\ldots,2\lfloor(J-1)/2\rfloor\}^{m-1}$.
Since $\lfloor(J-1)/2\rfloor\ge(16\varepsilon)^{-1}$, it holds that
\begin{equation}\label{eq:lower-packing-size}
    |\mathcal L|\ge(16\varepsilon)^{-(m-1)}.
\end{equation}
For $l\in\mathcal L$, alternative $I_l$ lowers only the cost of $a_l$ by
$\Delta=\rho_0\varepsilon^2/8=\varepsilon^2/[16(m-1)]$.
This preserves nonnegativity because
$\rho_s-\rho_{s-1}\ge\rho_0\varepsilon$ implies
$c(l)\ge3(m-1)\rho_0\varepsilon^2>\Delta$.
Let's now define
\begin{equation}\label{eq:lower-response-region}
\mathcal R_l=\left\{x:\sum_{i<m}
\left[\max_{0\le k\le J}\phi_k(x_i)-\phi_{l_i}(x_i)\right]\le\Delta\right\}.
\end{equation}
\begin{lemma}[Separated alternatives]\label{lem:lower-regions}
If $I_l$ selects $a_l$, then $x\in\mathcal R_l$. Moreover,
$\mathcal R_l\subseteq\prod_{i<m}[(l_i-1/8)\varepsilon,
(l_i+1+1/8)\varepsilon]$ and these regions hence are pairwise disjoint.
Outside $\mathcal R_l$, the alternative selects the baseline response and gives principal utility at most $1/2$. Some feasible contract uniquely induces $a_l$ and gives utility at least
$1/2+\gamma_m\varepsilon$, where $\gamma_m=3/[128(m-1)]$.
\end{lemma}

\begin{proof}
The baseline joint-action envelope is
$f_m+\sum_{i<m}\max_k\phi_k(x_i)$, whereas the cost reduction raises only $a_l$ by $\Delta$. Thus selecting $a_l$ requires~\eqref{eq:lower-response-region}.
Note that every deficit in that condition is nonnegative. If $x_i<l_i\varepsilon$, the comparison with level $l_i-1$ and
$\rho_{l_i}-\rho_{l_i-1}\ge\rho_0\varepsilon$ gives
$\rho_0\varepsilon(l_i\varepsilon-x_i)\le\Delta$, hence
$x_i\ge(l_i-1/8)\varepsilon$. Comparison with $l_i+1$ similarly gives the upper endpoint, and this neighbor exists because $l_i\le J-1$. Distinct even multi-indices differ by at least two in some coordinate, so their boxes, and therefore their regions, are disjoint. Outside $\mathcal R_l$, $a_l$ lies strictly below the unchanged baseline envelope, so the common priority rule then selects the baseline response, whose utility is at most $1/2$ by Lemma~\ref{lem:lower-baseline}.

For profitability, let
$s=\Delta/[2(\rho_{l_1}-\rho_{l_1-1})]\le\varepsilon/16$, and choose
$f_m=0$, $f_1=l_1\varepsilon-s$ and $f_i=l_i\varepsilon$ for $2\le i<m$.
This contract is feasible. The baseline deficit of $a_l$ is exactly $\Delta/2$, so after perturbation it uniquely beats every joint action, and the positive scalar envelope also beats the null action. Its principal utility is
\[
\sum_{i<m}\rho_{l_i}(1-f_i)=\tfrac12+\rho_{l_1}s
\ge\tfrac12+\frac{3\Delta}{8\varepsilon}
=\tfrac12+\gamma_m\varepsilon,
\]
where
$\rho_{l_1}/(\rho_{l_1}-\rho_{l_1-1})
=[1-(l_1-1)\varepsilon]/\varepsilon\ge3/(4\varepsilon)$.
\end{proof}

\begin{lemma}[Adaptive information bound]\label{lem:lower-information}
For any adaptive policy, let $\mathbb P_0^T,\mathbb P_l^T$ be its full interaction laws under the baseline and $I_l$, and let
$N_l=\sum_{t\le T}\mathbf1\{x(f_t)\in\mathcal R_l\}$ and
$n_l=\mathbb E_0N_l$. Then, it holds that
\begin{equation}\label{eq:lower-kl-chain}
D_{\rm KL}(\mathbb P_0^T\Vert\mathbb P_l^T)\le5\varepsilon^2n_l,
\qquad \sum_{l\in\mathcal L}n_l\le T.
\end{equation}
\end{lemma}

\begin{proof}
Inside $\mathcal R_l$, its positive lower box endpoint ensures that the baseline selects a joint action $a_k$. Scalar optimality and the box imply
$k_i\in\{l_i-1,l_i,l_i+1\}$. Adjacent slopes then satisfy
$0<\rho_j-\rho_{j-1}\le2\rho_0\varepsilon$. If the alternative response differs, it is $a_l$. Using
$D_{\rm KL}(p\Vert q)\le\chi^2(p,q)$,
$\rho_{l_i}\ge\rho_0$, and $p_m(l)\ge1/3$ gives
\[
D_{\rm KL}(p(k)\Vert p(l))
\le\sum_{i<m}\frac{(\rho_{k_i}-\rho_{l_i})^2}{\rho_{l_i}}
+\frac{[\sum_{i<m}(\rho_{k_i}-\rho_{l_i})]^2}{p_m(l)}
\le2\varepsilon^2+3\varepsilon^2.
\]
Outside $\mathcal R_l$, the response laws coincide. The adaptive KL chain rule
\citep[Lemma~15.1 and Exercise~15.8]{lattimore2020bandit} now proves the first inequality in~\eqref{eq:lower-kl-chain} and policy kernels cancel. Region disjointness gives
$\sum_lN_l\le T$ pathwise under the baseline and finishes the proof.
\end{proof}

\begin{proof}[Proof of Theorem~\ref{thm:lower}]
By Lemma~\ref{lem:lower-regions},
we know $\mathfrak R_T(\pi,I_l)\ge\gamma_m\varepsilon(T-\mathbb E_lN_l)$.
We independently sample a uniform round $\tau$. The Bretagnolle-Huber inequality
\citep[Theorem~14.2]{lattimore2020bandit} and Lemma~\ref{lem:lower-information} then imply
\[
\frac{n_l}{T}+1-\frac{\mathbb E_lN_l}{T}
=\mathbb P_0^T(x(f_\tau)\in\mathcal R_l)
+\mathbb P_l^T(x(f_\tau)\notin\mathcal R_l)
\ge\tfrac12e^{-5\varepsilon^2n_l}.
\]
Consequently, it holds that
$\mathfrak R_T(\pi,I_l)\ge\gamma_m\varepsilon[
(T/2)e^{-5\varepsilon^2n_l}-n_l]$.
Averaging over $l$ and using convexity together with
$\sum_l n_l\le T$, we obtain
\[
\max_{l\in\mathcal L}\mathfrak R_T(\pi,I_l)
\ge\gamma_m\varepsilon T\left[
\tfrac12\exp\!\left(-\frac{5T\varepsilon^2}{|\mathcal L|}\right)
-\frac1{|\mathcal L|}\right].
\]
We set $\varepsilon=64^{-1}T^{-1/(m+1)}$. By~\eqref{eq:lower-packing-size}, we have
$5T\varepsilon^2/|\mathcal L|\le1/4$ and
$1/|\mathcal L|\le1/4$. Since
$\tfrac12e^{-1/4}-\tfrac14>0$, the last display is at least
$c_mT^{m/(m+1)}$. Because $\pi$ is arbitrary and every $I_l$ belongs to $\mathcal I_m$, we have
\[
\mathfrak R_T^\star(m)
=\inf_\pi\sup_{I\in\mathcal I_m}\mathfrak R_T(\pi,I)
\ge
\inf_\pi\max_{l\in\mathcal L}\mathfrak R_T(\pi,I_l)
\ge c_mT^{m/(m+1)},
\]
which ends the proof.
\end{proof}

\newpage
\bibliographystyle{ims}
\bibliography{graphbib}

\end{document}